\documentclass[twoside,leqno,twocolumn]{article}

\usepackage[letterpaper]{geometry}
\usepackage{ltexpprt}

\usepackage{amsmath,amssymb,amsfonts}
\usepackage{caption}
\usepackage{graphicx}
\usepackage{booktabs}
\usepackage{algorithm}
\usepackage{algorithmic}
\usepackage{amsfonts}
\usepackage{amsmath, bm}
\usepackage{multirow}
\usepackage{bbm}
\usepackage{color}
\usepackage{subcaption}
\usepackage{adjustbox}
\usepackage{cite}
\usepackage{times}
\usepackage{url}
\begin{document}

\title{Achieving User-Side Fairness in Contextual Bandits}

\author{Wen Huang \and Kevin Labille
\and Xintao Wu \thanks{University of Arkansas, \{wenhuang, kclabill, xintaowu\}@uark.edu. The first two authors have equal contribution.} \and Dongwon Lee \thanks{Penn State University, dongwon@psu.edu} \and Neil Heffernan \thanks{Worcester Polytechnic Institute, nth@wpi.edu}}

\date{}

\maketitle

\begin{abstract}
Personalized recommendation based on multi-arm bandit (MAB) algorithms  has shown to lead to high utility and efficiency as it can dynamically adapt the recommendation strategy based on feedback. However, unfairness could incur in personalized recommendation. In this paper, we study how to achieve user-side fairness in personalized recommendation. We formulate our fair personalized recommendation as a modified contextual bandit and focus on achieving fairness on the individual whom is being recommended an item as opposed to achieving fairness on the items that are being recommended.  We introduce and define a metric that captures the fairness in terms of rewards received for both the privileged and protected groups. We develop a fair contextual bandit algorithm, Fair-LinUCB, that improves upon the traditional LinUCB algorithm to achieve group-level fairness of users. Our algorithm detects and monitors unfairness while it learns to recommend personalized videos to students to achieve high efficiency. We provide a theoretical regret analysis and show that our algorithm has a slightly higher regret bound than LinUCB. We conduct numerous experimental evaluations to compare the performances of our fair contextual bandit to that of LinUCB and show that our approach achieves group-level fairness while maintaining a high utility. 
\end{abstract}


\section{Introduction}

Personalized recommendation based on multi-arm bandit (MAB) algorithms  has become a popular topic of research and shown to lead to high utility and efficiency \cite{DBLP:conf/cec/BouneffoufRA20} as it dynamically adapts the recommendation strategy based on feedback. However, it is also known that such personalization could incur biases or even discrimination that can influence decisions and opinions
\cite{epstein2015search,farahat2012effective}. Recently researchers have started taking fairness and discrimination into consideration in the design of MAB based personalized recommendation algorithms \cite{cseli2018algorithmic,liu2017calibrated,zhu2018fairness}. However, they focused on the fairness of the recommended items (e.g., services provided by small or large companies) instead of the customers who received those items. For example, \cite{liu2017calibrated}  focused on individual fairness, i.e., “treating similar individuals similarly,” and considered the individual as an arm with the aim of ensuring the probability of selecting an arm is equal to the probability with which the arm has the best quality realization. \cite{cseli2018algorithmic} aimed to achieve group fairness over items by ensuring the probability distribution from which items are sampled satisfies certain fairness constraints at all time steps. 
In this paper, we aim to develop novel algorithms to ensure fair and ethical treatment of customers with different profile attributes (e.g., gender, race, education, disability, and economic conditions) in a contextual bandit based personalized recommendation. 

\begin{table}[!h]
\centering
\caption{Illustrative example}
\begin{subtable}[h]{0.5\textwidth}
\centering
\begin{tabular}{c|c|c|c|c} 
    \hline
    Student & Gender& Grade  & GPA &...  \\
    \hline
Alice & female & 9th  & 2.6  & ... \\
    \hline
Bob &  male & 9th  &  2.6   &...\\
    \hline
... &  ... & ...  &  ...   &...\\
    \hline
\end{tabular}
\caption{Students}
\label{ex_table_stu}
\end{subtable}
\begin{subtable}[h]{0.5\textwidth}
\centering
\begin{tabular}{c|c|c|c|c}
    \hline
    Video & Gender of speaker & rating & length &...\\
    \hline
2501 &  female & 4.3 &  4 minutes &...\\
    \hline
0964  & male &  4.3 &  6 minutes  &...\\
    \hline
 ... & ... &  ... &  ... &...\\
    \hline  
\end{tabular}
\caption{Videos}
\label{ex_table_vid}
\end{subtable}
\begin{subtable}[h]{0.5\textwidth}
\centering
\begin{tabular}{c|c|c}
    \hline
    Student & Video & Reward\\
    \hline
Alice & 2501 & 0.60\\
    \hline
Bob & 0964& 0.80 \\
    \hline
... & ...& ... \\
    \hline
\end{tabular}
\caption{Recommendations}
\label{ex_table_rec}
\end{subtable}
\label{prob_example}
\end{table}

Consider the personalized educational video recommendation in Table \ref{ex_table_rec} as an illustrative example.  Table \ref{ex_table_stu} shows two students, Alice and Bob, having the same profiles except for the gender. Table \ref{ex_table_vid} shows potential videos and Table \ref{ex_table_rec} shows recommendations by a personalized recommendation algorithm. 
Focusing on the fairness of the video would ensure that videos featuring female speakers have similar chances of being recommended as those featuring male speakers. However, one group of students could benefit more from the recommended videos than the other group, therefore yielding to an unequal improvement of the learning performances. In our work, rather than focusing on the fairness of the item being recommended, i.e., the video, we focus on the user-side fairness in terms  of the reward, i.e., the improvement of student's learning performance after watching the recommended video. We want to ensure that both male students and female students who share similar profiles will receive a similar reward regardless of the video being recommended, such that they both benefit from the video recommendations and improve their learning performance equally.

We study how to achieve the user-side fairness in the classic contextual bandit algorithm. The contextual bandit framework \cite{langford2007epoch}, which is used to sequentially recommend items to a customer based on her contextual information, is able to fit user preferences, address the cold-start problem by balancing the exploration and exploitation trade-off in recommendation systems, and  simultaneously adapt the recommendation strategy based on feedback to maximize the customer's learning performance. However, such a personalized recommendation system could induce an unfair treatment of certain customers which could lead to discrimination. We develop a novel fairness aware contextual bandit algorithm such that customers will be treated fairly in personalized learning. 

We train our fair contextual bandit algorithm to detect discrimination, that is, whether or not a group of customers is being privileged in terms of reward received. Our fair contextual bandit algorithm then measures to what degree each of the items (arms in bandits) contributes to the discrimination. Furthermore, in order to counter the discrimination, if any, we introduce a fairness penalty factor. The goal of this penalty factor is to maintain a balance between fairness and utility, by ensuring that the arm picking strategy will not incur discrimination whilst achieving good utility. Finally, we compare our algorithm against the traditional LinUCB both theoretically and empirically and we show that our approach not only achieves group-level fairness in terms of reward, but also yields comparable effectiveness. 

Overall, our contributions are two-fold. First, we develop a fairness aware contextual bandit algorithm that achieves user-side fairness in terms of reward and is robust against factors that would otherwise increase or incur discrimination. Secondly, we provide a theoretical regret analysis to show that our algorithm has a regret bound higher than LinUCB up to only an additive constant. 

\section{Related Work}
\subsection{Bandits based Recommendation}
Many bandits based algorithms have been developed to suggest recommendations for products and services. Contextual bandit \cite{langford2007epoch} is an extension of the classic multi-armed bandit (MAB) algorithm \cite{katehakis1987multi}. The MAB chooses an action from a fixed set of choices to maximize the expected gain where each choice’s properties are only partially known at the time of choice and the gain of a choice will be observed only after the action is taken. In other words, the MAB simultaneously attempts to acquire new information (exploration) and optimize decisions based on existing knowledge (exploitation). Compared to the traditional content-based recommendation approaches, the MAB is able to fit dynamic-changed user preferences over time and address the cold-start problem by balancing the exploration and exploitation trade-off in the recommendation system. However, the MAB  does not use any information about the state of the environment. The contextual bandit model extends the MAB model by making the recommendation conditional on the state of the environment. Other variations  include stochastic \cite{abbasi2011improved}, Bayesian \cite{DBLP:conf/atal/GhalmeJGN17}, adversarial \cite{DBLP:conf/icml/SyrgkanisKS16}, and non-stationary \cite{DBLP:conf/nips/GurZB14} bandits. In this paper, we focus on the contextual bandit model because it is posed to help identify which items work for whom. 
The contextual information is the customer's features and the features of the items under exploration, and the reward is derived from purchase record or customer feedback.

\subsection{Fairness-aware Machine Learning}
Fairness aware machine learning is receiving increased attention. Discrimination is unfair treatment towards individuals based on the group to which they are perceived to belong. In machine learning, training data may have historically biased decisions against the protected group; models learned from such data may make discriminatory predictions against the protected group. The fair learning research community has developed extensive fair machine learning algorithms based on a variety of fairness metrics, e.g., equality of opportunity and equalized odds \cite{hardt2016equality,zafar2017fairness}, direct and indirect discrimination \cite{zhang2017causal,zhang2018fairness, chiappa2018path}, counterfactual fairness \cite{kusner2017counterfactual,russell2017worlds,wu2019counterfactual}, and path-specific counterfactual fairness \cite{wu2019pcfairness}.

Recently researchers have started taking fairness and discrimination into consideration in the design of MAB based personalized recommendation algorithms \cite{DBLP:conf/nips/JosephKMR16,DBLP:conf/aies/JosephKMNR18,DBLP:conf/icml/JabbariJKMR17,cseli2018algorithmic,liu2017calibrated,zhu2018fairness, burke2017multisided, burke2018balanced, ekstrand2018exploring}.
Among them, \cite{DBLP:conf/nips/JosephKMR16} was the first paper of studying fairness in classic and contextual bandits. It defined fairness with respect to one-step rewards introduced a notion of meritocratic fairness, i.e., the algorithm should never place higher selection probability on a less qualified arm (e.g., job applicant) than on a more qualified arm. This was inspired by equal treatment, i.e., similar people should be treated similarly. The following works along this direction include \cite{DBLP:conf/aies/JosephKMNR18} for infinite and contextual bandits, \cite{DBLP:conf/icml/JabbariJKMR17} for reinforcement learning,  \cite{liu2017calibrated} for the simple stochastic bandit setting with calibration based fairness.  In \cite{DBLP:journals/corr/abs-1907-10516}, the authors studied the problem of learning fair stochastic multi-armed bandit where each arm is required to be pulled for at least a given fraction of the total available rounds. In \cite{DBLP:conf/infocom/LiLJ19}, the authors studied fairness in the setting that multiple arms can be simultaneously played and an arm could sometimes be sleeping. 
\cite{DBLP:conf/nips/GillenJKR18} used an unknown Mahalanobis similarity metric from some weak feedback that identifies fairness violations through an oracle rather than adopting a quantitative fairness metric over individuals. The fairness constraint requires that the difference between the probabilities that any two actions are taken is bounded by the distance between their contexts. All the above papers require some fairness constraint on arms at every round of the learning process, which is different from our user-side fairness setting. 
How to achieve fairness in other related contexts have also been studied, e.g.,  sequential decision making \cite{DBLP:conf/ijcai/HeidariK18}, online stochastic classification \cite{DBLP:journals/corr/abs-1908-07009},  offline contextual bandits \cite{DBLP:conf/nips/MetevierGBKBBT19}, and collaborative filtering based recommendation systems \cite{DBLP:conf/nips/YaoH17,DBLP:conf/fat/EkstrandTAEAMP18}.

\section{Preliminary}\label{ch3}

Throughout this paper, we use bold letters to denote a vector.  We use $||\mathbf{x}||_2$ to define the L-2 norm of a vector $\mathbf{x} \in \mathbb{R}^d$. For a positive definite matrix $A \in \mathbb{R}^{d \times d}$, we define the weighted 2-norm of $\mathbf{x} \in \mathbb{R}^d$ to be $||\mathbf{x}||_A = \sqrt{\mathbf{x}^\mathrm{T}A\mathbf{x}}$.

\subsection{LinUCB Algorithm}
We use the linear contextual bandit \cite{chu2011contextual} as one baseline model for our personalized recommendation. In the linear contextual bandit, the reward for each action is an unknown linear function of the contexts. 
Formally, we model the personalized  recommendation as a contextual multi-armed bandit problem, where each user $u$ is a “bandit player”, each potential item $a \in \mathcal{A}$ is an arm and $k$ is the number of item candidates. 
At time $t$, there is a coming user $u$. For each item $a \in \mathcal{A}$, its contextual feature vector $\mathbf{x}_{t,a} \in \mathbb{R}^d$  represents the concatenation of the user and the item feature vectors. The algorithm takes all contextual feature vectors as input, recommends an item $a_t \in \mathcal{A}$ and observes the reward $r_{t,a_t}$, and then updates its item recommendation strategy with the new observation $(\mathbf{x}_{t,a_t}, a_t,r_{t,a_t})$.
During the learning process, the algorithm does not observe the reward information for unchosen items.

The total reward by round $t$ is defined as $\sum_t r_{t,a_t}$ and the optimal expected reward as $\mathbb{E}[\sum_t r_{t,a^*}]$, where $a^*$ indicates the best item that can achieve the maximum reward at time $t$. We aim to train an algorithm so that the maximum total reward can be achieved. Equivalently, the algorithm aims to minimize the regret $R(T)=\mathbb{E}[\sum_t r_{t,a^*}]-\mathbb{E}[\sum_t r_{t,a_t}]$. 
The contextual bandit algorithm balances exploration and exploitation to minimize regret since there is always uncertainty about the user’s reward given the specific item. 
 
\begin{algorithm}[h]
	\caption{LinUCB}
	\begin{algorithmic}[1]
		\STATE \text{Input:} $\alpha$ $\in$ $\mathbb{R}^+$
	    \FOR {t = 1,2,3,...,$T$} 
	    \STATE Observe contextual features of all arms $a \in \mathcal{A}_t : \mathbf{x}_{t,a} \in \mathbb{R}^d $ 
	    \FOR {$a \in \mathcal{A}_t$ }
	    \IF  {$a$ is new}
    	\STATE $A_a \leftarrow \mathbf{I_d}$ (d-Dimension identity matrix)
    	\STATE $\mathbf{b}_a \leftarrow \mathbf{0}_{d \times 1}$ (d-Dimension zero vector)
    	\ENDIF
	    \smallskip
	    \STATE $\bm{\hat{\theta}}_a \leftarrow A^{-1}_a \mathbf{b}_a$
	    \STATE $p_{t,a} \leftarrow \hat{\bm{\theta}}_a^\mathrm{T}\mathbf{x}_{t,a} + \alpha \sqrt{\mathbf{x}_{t,a}^\mathrm{T} A_a^{-1} \mathbf{x}_{t,a}} $
	    \ENDFOR
	    \STATE Choose arm $a_t = argmax_{a \in \mathcal{A}_t} p_{t,a}$ with ties broken arbitrarily, and observe a real-valued payoff $r_{t,a_t} $
	    \STATE $A_{a_t} \leftarrow A_{a_t} +\mathbf{x}_{t,a_t}\mathbf{x}^\mathrm{T}_{t,a_t} $
	    \STATE $\mathbf{b}_{a_t} \leftarrow \mathbf{b}_{a_t} + r_{t,a_t}\mathbf{x}_{t,a_t} $
	    \ENDFOR
	\end{algorithmic}
	\label{alg:linucb}
\end{algorithm}

We adopt the idea of upper confidence bound (UCB) for our personalized recommendation. Algorithm 1 shows the LinUCB algorithm as introduced by \cite{li2010contextual}. It assumes the expected reward   is linear in its $d$-dimensional features $\mathbf{x}_{t,a}$ with some unknown coefficient vector $\bm{\theta}^*_a$. Formally, for all $t$, we have the expected reward at time $t$ with arm $a$ as $\mathbb{E}[r_{t,a} |\mathbf{x}_{t,a}]= \bm{\theta}_a^{*\mathrm{T}}\mathbf{x}_{t,a}$.
Here the dot product of $\bm{\theta}^*_a$ and $\mathbf{x}_{t,a}$ could also be succinctly expressed as $\langle \bm{\theta}^*_a,\mathbf{x}_{t,a} \rangle$. At each round $t$, we observe the realized reward $r_{t,a} = \langle \bm{\theta}^*_a,\mathbf{x}_{t,a} \rangle + \epsilon_t$
where $\epsilon_t$ is the noise term.

Basically, LinUCB applies ridge regression technique to estimate the true coefficients. Let $D_a \in \mathbb{R}^{m_a \times d}$ denote the context of the historical observations when arm $a$ is selected and $\mathbf{r}_a \in \mathbb{R}^{m_a}$ denote the relative rewards. 
The regularised least-square estimator for $\bm{\theta}_a$ could be expressed as:
\begin{equation}
\bm{\hat{\theta}}_a = \mathop{\arg\min}_{\bm{\theta} \in \mathbb{R}^d} \left(  \sum_{i=1}^{m_a}(r_{i,a} - \langle \bm{\theta},D_a{(i,:)}\rangle)^2 + \lambda ||\bm{\theta}||^2_2  \right)
\label{ridge}
\end{equation}
where $\lambda$ is the penalty factor of the ridge regression. The solution to Equation $\ref{ridge}$ is: 
\begin{equation}
    \bm{\hat{\theta}}_a = (D_a^\mathrm{T}D_a +\lambda I_d)^{-1} D_a^\mathrm{T} \mathbf{r}_a
\end{equation}

\cite{li2010contextual} derived a confidence interval that contains the true expected reward with probability at least $1-\delta$:
\[ \left| \hat{\bm{\theta}}_a^\mathrm{T}\mathbf{x}_{t,a} - \mathbb{E}[r_{t,a}|\mathbf{x}_{t,a}] \right| \leq  \alpha\sqrt{ \mathbf{x}_{t,a}^\mathrm{T} (D_a^\mathrm{T}D_a +\lambda I_d) \mathbf{x}_{t,a}  }\]
for any $\delta > 0 $, where $\alpha = 1 + \sqrt{ln(2/\delta)/2}$ . Following the rule of optimism in the face of uncertainty for linear bandits (OFUL), this confidence bound leads to a reasonable arm-selection strategy: at each round $t$, pick an arm by
\begin{equation}
a_t = argmax_{a \in \mathcal{A}_t} \left( \bm{\hat{\theta}}_a^\mathrm{T}\mathbf{x}_{t,a} + \alpha \sqrt{\mathbf{x}_{t,a}^\mathrm{T} A_a^{-1} \mathbf{x}_{t,a}}    \right)
\end{equation}
where $A_a = D_a^\mathrm{T} D_a +  \lambda I_d$. The parameter $\lambda$ could be tuned to a suitable value in order to improve the algorithm's performance. Line 13 and 14 in Algorithm \ref{alg:linucb} provide an iterative way to update the arm-related matrices $A_a$ and $b_a$. In the remaining content we will denote the weighted 2-norm $\sqrt{\mathbf{x}_{t,a}^\mathrm{T} A_a^{-1} \mathbf{x}_{t,a}} $ as $|| \mathbf{x}_{t,a}||_{A_a^{-1}}$ for the sake of simplicity.

\subsection{Regret Bound of LinUCB}
\label{sec:regret_bound}

Existing research works (e.g., \cite{abbasi2011improved,wu2016contextual}) on deriving the regret bound of LinUCB are based on the following four assumptions:
\begin{enumerate}
    \item The true coefficient $\bm{\theta}^*$ is shared by all arms.
    \item The error term $\epsilon_t$ follows 1-sub-Gaussian distribution for each time point. 
    \item$\{\alpha_t\}^n_{i=1} $ is a non-decreasing sequence with $\alpha_1 \geq 1$.
    \item $||\mathbf{x}_{t,a}||_2 < L $,  $||\bm{\theta}^*||_2 < M $ for all time points and arms.
\end{enumerate}
For assumption 1, since there is only one unified $\bm{\theta}$, we change the notation of $D_a$, $\mathbf{r}_a$ to $D_t$ and $\mathbf{r}_t$ to denote the historical observations up to time $t$ for all arms. The matrix $A_a$ will be denoted as $A_t$ accordingly. For assumption 3, following \cite{abbasi2011improved} and \cite{wu2016contextual}, we modify $\alpha$ in Algorithm $\ref{alg:linucb}$ to be a time dependent sequence to get a suitable confidence set for $\bm{\theta}^*$ at each round, but use a fixed and tuned $\alpha$ in the experiment part to make the online computation more efficient.

To derive the regret bound, the first step is to construct a confidence set $\mathcal{C}_t \in \mathbb{R}^d $ for the true coefficient. At each round $t$, a natural choice is to make $\mathcal{C}_t$ centered at $\bm{\hat{\theta}}_{t-1}$. \cite{abbasi2011improved} shows that the confidence ellipsoid could be a suitable choice for constructing the confidence region, which is defined as follows:
\[\mathcal{C}_t  = \{ \bm{\theta} \in \mathbb{R}^d : ||\bm{\theta} - \bm{\hat{\theta}}_{t-1} ||_{A_{t-1}} < \alpha_t \} \]

The key point is how to obtain an appropriate $\alpha_t$ at each round to make $\mathcal{C}_t$ contain the true parameter $\bm{\theta}^*$ with high probability and be as small as possible simultaneously. \cite{abbasi2011improved} takes the advantages of the martingale techniques and derives a confidence bound in terms of the weighted 2-norm shown in Lemma \ref{abbasi}. 
\begin{lemma}
(Theorem 2 in \cite{abbasi2011improved}) Suppose the noise term is 1-sub-Gaussian distributed, let $\delta \in (0,1)$, with probability at least $1-\delta$, it holds that for all $t \in \mathbb{N^+}$,
\begin{equation}
\begin{split}
 ||\bm{\theta}^* - \bm{\hat{\theta}}_t ||_{A_t} \leq \sqrt{\lambda}||\bm{\theta}^*||_2 +  \sqrt{2log(|A_t|^{1/2}|\lambda I_d|^{-1/2}\delta^{-1}) }
\end{split}
\label{confidence_bound}
\end{equation}
\label{abbasi}
\end{lemma}
The RHS of Equation \ref{confidence_bound} gives an appropriate selection of $\alpha_t$ for the confidence ellipsoid. 
Under the fact that $\theta ^* \in \mathcal{C}_t$ and the optimistic arm selection rule of LinUCB we could further bound the regret at each round with high probability by 
$r_t =  \langle \bm{\theta}^*,\mathbf{x}_{t,a} \rangle - \langle \bm{\hat{\theta}},\mathbf{x}_{t,a}\rangle \leq 2 \alpha_t ||\mathbf{x}_{t,a} ||_{A^{-1}_t}$. Summing up the regret at each round, the following corollary gives a $\tilde{\mathcal{O}}(dlog(T))$ cumulative regret bound up to time $T$.

\begin{corollary}
(Corollary 19.3 in \cite{lattimore2018bandit}) Under the assumptions above, the expected regret of LinUCB with $\delta = 1/T $ is bounded by 
\begin{equation}
\begin{split}
    R_T \leq Cd\sqrt{T}log(TL)
\end{split}
\end{equation}
where $C$ is a suitably large constant.
\label{simple_bound}
\end{corollary}

\section{Fairness Aware Contextual Bandits}
We focus on how to achieve user-side fairness in contextual bandit based recommendation and  present our fair contextual bandit algorithm, called Fair-LinUCB and derive its regret bound.

\subsection{Problem Formulation} \label{prob}

We define a sensitive attribute $S \in \mathbf{x}_{t,a}$ with domain values $\{s^+ , s^-\}$ where $s^+$ ($s^-$) is the value of the privileged (protected) group. Let $T_s$ denote a time index subset such that the users being treated at time points in $T_s$ all hold the same sensitive attribute value $s$.  
 We introduce the group-level cumulative mean reward as  $\bar{r}^{s} =\dfrac{1}{|T_s|}\sum_{t \in T_s} r_{t,a}$. 
Specifically, $\bar{r}^{s^+}$ denotes the cumulative mean reward of the individuals with sensitive attribute $S=s^+$, and  $\bar{r}^{s^-}$ denotes the cumulative mean reward of all individuals having the sensitive attribute $S=s^-$.

We define the group fairness in contextual bandits as $\mathbb{E}[\bar{r}^{s^+}]=\mathbb{E}[\bar{r}^{s^{-}}]$, more specifically, the expected mean reward of the protected group and that of the unprotected group should be equal. 
A recommendation algorithm incurs group-level unfairness in regards to a sensitive attribute $S$ if $|\mathbb{E}[\bar{r}^{s^+}]-\mathbb{E}[\bar{r}^{s^-}]|>\tau$ where $\tau \in \mathbb{R^+}$ reflects the tolerance degree of unfairness. 

\subsection{Fair-LinUCB algorithm} \label{fairnessCB}
We describe our fair LinUCB algorithm and show its pseudo code in Algorithm \ref{alg:fairucb}. The key difference from the traditional LinUCB is the strategy of choosing an arm during recommendation (shown in Line 12 of Algorithm  \ref{alg:fairucb}). In the remaining of this section, we explain how this new strategy achieves user-side group-level fairness.  

\begin{algorithm}
	\caption{Fair-LinUCB}
	\begin{algorithmic}[1]
	    \STATE \text{Input:} $\alpha$ , $\gamma$ $\in$ $\mathbb{R^+}$
	    \STATE $\bar{r}^{s^+}, \bar{r}^{s^-} \leftarrow 0 $
	    \FOR {t = 1,2,3,..., T} 
	    \STATE Observe features of all arms $a \in \mathcal{A}_t : \mathbf{x}_{t,a} \in \mathbb{R}^d $ 
	    \FOR {$a \in \mathcal{A}_t$ }
	    \IF{$a$ is new}
	    \STATE $A_a \leftarrow \lambda \mathbf{I}_d$ (d-Dimension identity matrix)
	    \STATE $\mathbf{b}_a \leftarrow \mathbf{0}_{d \times 1}$ (d-Dimension zero vector)
        \STATE $\bar{r}_a^{s^+}, \bar{r}_a^{s^-} \leftarrow 0 $
	    \ENDIF
	    \STATE $\bm{\hat{\theta}}_a \leftarrow A_a^{-1}\mathbf{b}_a   $
	    \STATE $p_{t,a} \leftarrow \bm{\hat{\theta}}_a^\mathrm{T}\mathbf{x}_{t,a} + \alpha || \mathbf{x}_{t,a}||_{A_a^{-1}} + \mathcal{L}(\gamma, F_a)$
	    \smallskip
	    \ENDFOR
	    \STATE Choose arm $a_t = argmax_{a \in \mathcal{A}_t} p_{t,a}$ with ties broken arbitrarily, and observe a real-valued payoff $r_{t,a_t} $
	    \STATE $ A_a \leftarrow A_a +\mathbf{x}_{t,a_t}\mathbf{x}^\mathrm{T}_{t,a_t} $
	    \STATE $\mathbf{b}_a \leftarrow \mathbf{b}_a + r_{t,a_t}\mathbf{x}_{t,a_t} $
        \IF{$S_t = s^+$}
        \STATE update $\bar{r}^{s^+},\bar{r}_a^{s^+}$ with $r_{t,a_t}$
	    \ELSE  
	    \STATE update $\bar{r}^{s^-},\bar{r}_a^{s^-}$ with $r_{t,a_t}$
        \ENDIF
	    \ENDFOR
	\end{algorithmic}
	\label{alg:fairucb}
\end{algorithm}

Given a sensitive attribute $S$ with domain values $\{s^+, s^-\}$, the goal of our fair contextual bandit is to minimize the cumulative mean reward difference between the protected group and the privileged group while preserving its efficiency. Note that Fair-LinUCB can be extended to the general setting of multiple sensitive attributes $S_j \in \bm{S} = \{ S_1,S_2, ... , S_l \}$ where $\bm{S} \subset \mathbf{x}_{t,a}$ and each $S_j$ can have multiple domain values. In order to measure the unfairness at the group-level, our Fair-LinUCB algorithm will keep track of both cumulative mean rewards along the time, e.g., $\bar{r}^{s^+}$ and $\bar{r}^{s^-}$. We capture the orientation of the bias (i.e., towards which group the bias is leaning) through the sign of the cumulative mean reward difference.
By doing so, Fair-LinUCB is able to know which group is being discriminated and which group is being privileged. 

When running context bandits for recommendation, each arm may generate a reward discrepancy and therefore contribute to the unfairness to some degree. Fair-LinUCB captures the reward discrepancy at the arm level by keeping track of the cumulative mean reward generated by each arm $a$ for both groups $s^+$ and $s^-$. Specifically, let $\bar{r}^{s^+}_a$ denote the average of the rewards generated by arm $a$ for the group $s^+$, and let $\bar{r}^{s^-}_a$ denote the average of the rewards generated by arm $a$ for the group $s^-$.
The bias of an arm is thus the difference of both averages: $\Delta_a = (\bar{r}^{s^+}_a - \bar{r}^{s^-}_a)$. Finally, by combining the direction of the bias and the amount of the bias induced by each arm $a$, we define the fairness penalty term as $F_a =  -sign(\bar{r}^{s^+} - \bar{r}^{s^-}) \cdot \Delta_a$, and exert onto the UCB value in our fair contextual bandit algorithm. Note that the lesser an arm contributes to the bias, the smaller the penalty. 

As a result, if an arm has a high UCB but incurs bias, its adjusted UCB value will decrease and it will be less likely to be picked by the algorithm. In contrast, if an arm has a small UCB but is fair, its adjusted UCB value will increase, and it will be more likely to be picked by the algorithm, thereby reducing the potential unfairness in recommendation. Different from  the traditional LinUCB that picks the arm to solely maximize the UCB, our Fair-LinUCB  accounts for the fairness of the arm and picks the arm that maximizes the summation of the UCB and the fairness. 
Formally, we show the modified arm selection criteria in Equation \ref{fairpta}.
\begin{equation}
p_{t,a} \leftarrow \bm{\hat{\theta}}_a^\mathrm{T}\mathbf{x}_{t,a} + \alpha || \mathbf{x}_{t,a}||_{A_a^{-1}} + \mathcal{L}(\gamma, F_a)
\label{fairpta}
\end{equation}
We adopt a linear mapping function $\mathcal{L}$ with input parameters $\gamma$  and $F_a$  to transform the fairness penalty term proportionally to the size of its confidence interval. Specifically, 
\begin{align}
\mathcal{L}(\gamma, F_a) &= \frac{\alpha_t ||\mathbf{x}_{t,a_m} ||_{A^{-1}_t}}{2}(F_a+1)\gamma\label{mapping_function}\\
        a_m &= argmin_{a \in \mathcal{A}_t} || \mathbf{x}_{t,a}||_{A_a^{-1}}
\end{align}

Assuming that the reward generated is in the range $[0, 1]$, the fairness penalty $F_a$ lies in $[-1, 1]$.  When designing the coefficient of the linear mapping function, we choose $a_m$ to be the arm with the smallest confidence interval to guarantee a unified fairness calibration among all the arms. Under the effect of $\mathcal{L}$, the range of the fairness penalty is mapped from $[-1, 1]$ to $[0,\
\gamma \alpha_t ||\mathbf{x}_{t,a_m} ||_{A^{-1}_t}]$, which implies a similar scale with the confidence interval. In our empirical evaluations, we show how $\gamma$ controls fairness-accuracy trade-off on the practical performance of Fair-LinUCB.

\subsection{Regret Analysis}
\label{regret_analysis}
In this section, We prove that our Fair-LinUCB algorithm has a $\tilde{\mathcal{O}}(dlog(T))$ regret bound under certain assumptions with carefully chosen parameters. We adopt the regret analysis framework of linear contextual bandit and introduce a mapping function on the fairness penalty term. By applying the mapping function $\mathcal{L}$ we make our fairness penalty term possess the similar scale with the half length of the confidence interval. Thus we can merge the regret generated by UCB term and fairness term together and derive our regret bound.

\begin{theorem}
\label{theorem1}
Under the same assumptions shown in Section \ref{sec:regret_bound}, further assuming $\gamma$ is a moderate small constant with $\gamma \leq \Gamma$, there exists  $\delta \in (0,1)$ such that with probability at least $1-\delta$ Fair-LinUCB  achieves the following regret bound:
\begin{equation}
\begin{split}
    R_T &\leq  \sqrt{2Tdlog(1+TL^2/(d\lambda))} ~~ \times \\ 
    &(2+\Gamma)(\sqrt{\lambda}M + \sqrt{2log(1/\delta) + dlog(1+T L^2/(d\lambda))}) 
\end{split}    
\end{equation}
\label{th:regret_bound}
\end{theorem}
\begin{proof}
We first introduce three technical lemmas from \cite{abbasi2011improved} and \cite{lattimore2018bandit} to help us complete the proof of Theorem \ref{th:regret_bound}.

\begin{lemma}
(Lemma 11 in appendix of \cite{abbasi2011improved}) If $\lambda \geq max(1,L^2)$, the weighted L2-norm of feature vector could be bounded by :
$\sum_{t=1}^{T}||\mathbf{x}_{t,a}||^2_{A^{-1}_t}  \leq 2 log \frac{|A_t|}{\lambda^d}$
\label{zhou_lemma1}
\end{lemma} 

\begin{lemma}
(Lemma 10 in appendix of \cite{abbasi2011improved} ) The determinant $|A_t|$ could be bounded by:
$|A_t| \leq (\lambda + t L^2/d)^d$.
\label{zhou_lemma2}
\end{lemma}

\begin{lemma}
 (Theorem 20.5 in \cite{lattimore2018bandit}) With probability at least $1-\delta$, for all the time point $t \in \mathbb{N}^+$ the true coefficient $
\bm{\theta}^*$ lies in the set:
\begin{equation}
\begin{split}
\mathcal{C}_t  = \{ &\bm{\theta} \in \mathbb{R}^d :||\bm{\hat{\theta}}_t - \bm{\theta}||_{A_t} \leq \\
&\sqrt{\lambda}M + \sqrt{2log(1/\delta) + dlog(1+T L^2/(d\lambda))} \} 
\end{split}
\label{simplified_bound}
\end{equation}
\label{alpha_T}
\end{lemma}

In Fair-LinUCB, the range of fairness term is $[-
1,1]$, we apply a linear mapping function 
$\mathcal{L}(\gamma, x) = \frac{\alpha_t ||\mathbf{x}_{t,a_m} ||_{A^{-1}_t}}{2}(x+1)\gamma$
to map the range of $\mathcal{L}(\gamma, F_a)$ to $[0,\
\gamma \alpha_t ||\mathbf{x}_{t,a_m} ||_{A^{-1}_t}]$, where $a_m = argmin_{a \in \mathcal{A}_t} || \mathbf{x}_{t,a}||_{A_a^{-1}}  $.

According to the  rule, the regret at each time $t$ is bounded by: 
\begin{gather*}
\begin{split}
reg_t  
& = \mathbf{x}_{t,a}^\mathrm{T}\bm{\hat{\theta}}_t -
\mathbf{x}^\mathrm{T}_{t,a}\bm{\theta}^* \\
&\leq \mathbf{x}_{t,a}^\mathrm{T}\bm{\hat{\theta}}_t + \alpha_t||\mathbf{x}_{t,a}||_{A^{-1}_t} +   \mathcal{L}(\gamma, F_{a}) - \mathbf{x}^\mathrm{T}_{t,a}\bm{\theta}^*\\
&\leq \mathbf{x}_{t,a}^\mathrm{T}\bm{\hat{\theta}}_t + \alpha_t||\mathbf{x}_{t,a}||_{A^{-1}_t} +  \mathcal{L}(\gamma, F_{a})\\
&- (\mathbf{x}^\mathrm{T}_{t,a}\bm{\hat{\theta}_t} -\alpha_t||\mathbf{x}_{t,a}||_{A^{-1}_t})   \\
&\leq 2\alpha_t||\mathbf{x}_{t,a}||_{A^{-1}_t} + \mathcal{L}(\gamma,1) \\
& =  2\alpha_t||\mathbf{x}_{t,a}||_{A^{-1}_t} +  \gamma\alpha_t||\mathbf{x}_{t,a_m}||_{A^{-1}_t}\\
&\leq 2\alpha_t||\mathbf{x}_{t,a}||_{A^{-1}_t} + \gamma\alpha_t||\mathbf{x}_{t,a}||_{A^{-1}_t}\\
&\leq(2+\Gamma) \alpha_t||\mathbf{x}_{t,a}||_{A^{-1}_t}
\end{split}
\label{eq:r_t}
\end{gather*}

 The second line above is derived based on the theoretic result in Lemma \ref{abbasi} and following the selection rule of the Fair-LinUCB algorithm, specifically, $\mathbf{x}^\mathrm{T}_{t,a^*}\bm{\theta}^* \leq \mathbf{x}_{t,a^*}^\mathrm{T}\bm{\hat{\theta}}_t + \alpha_t||\mathbf{x}_{t,a^*}||_{A^{-1}_t} \leq \mathbf{x}_{t,a^*}^\mathrm{T}\bm{\hat{\theta}}_t + \alpha_t||\mathbf{x}_{t,a^*}||_{A^{-1}_t} + \mathcal{L}(\gamma, F_{a^*}) \leq \mathbf{x}_{t,a}^\mathrm{T}\bm{\hat{\theta}}_t + \alpha_t||\mathbf{x}_{t,a}||_{A^{-1}_t} + \mathcal{L}(\gamma, F_a)$.
Note that Lemma \ref{abbasi} can be equally applied here because the estimator $\hat{\theta}_t$ is still a valid ridge regression estimator at each round.

Summing up the regret at each bound, with probability at least $1-\delta$ the cumulative regret up to time $T$ is bounded by:
\begin{equation}
\begin{split}
R_T &= \sum_{t=1}^{T}reg_t \leq \sqrt{T\sum_{t=1}^{T}reg_t^2}\\
&\leq (2+\Gamma)\alpha_T\sqrt{T\sum_{t=1}^{T}||\mathbf{x}_{t,a}||^2_{A^{-1}_t}}  
\end{split}
\label{eq:R_t}
\end{equation}

Since $\{\alpha_t\}^n_{i=1}$ is a non-decreasing sequence, we can enlarge each element $\alpha_t$ to $\alpha_T$ to obtain the inequalities in Equation \ref{eq:R_t}.
By applying the inequalities from Lemma \ref{zhou_lemma1} and \ref{zhou_lemma2} we could further relax the regret bound up to time $T$ to:
\begin{equation}
\begin{split}
R_T &\leq (2+\Gamma)\alpha_T\sqrt{2Tlog\frac{|A_t|}{\lambda^d}}\\
&\leq (2+\Gamma)\alpha_T\sqrt{2Td(log(\lambda+TL^2/d)-log\lambda)}\\
&= (2+\Gamma)\alpha_T\sqrt{2Tdlog(1+TL^2/(d\lambda))}
\end{split}
\label{eq:R_t_further}
\end{equation}

Following the result of Lemma \ref{abbasi}, by loosing the determinant of $A_t$ according to Lemma \ref{zhou_lemma2},
Lemma \ref{alpha_T} provides a suitable choice for $\alpha_T$ up to time $T$. By plugging in the RHS from Equation \ref{simplified_bound} we get the regret bound shown in Theorem \ref{theorem1}:
\begin{gather*}
\begin{split}
    R_T &\leq  \sqrt{2Tdlog(1+TL^2/(d\lambda))} ~~ \times \\ 
    &(2+\Gamma)(\sqrt{\lambda}M + \sqrt{2log(1/\delta) + dlog(1+T L^2/(d\lambda))})
\end{split}
\end{gather*}
\end{proof}

\begin{corollary}
Setting $\delta = 1/T$, the regret bound in Theorem \ref{theorem1} could be simplified as $R_T \leq C' d\sqrt{T}log(TL)$.
\label{our_simplified_bound}
\end{corollary}
Comparing Corollary \ref{our_simplified_bound}  with Corollary \ref{simple_bound} (for LinUCB), we can see the regret bound of Fair-LinUCB is worse than the original LinUCB only up to an additive constant. This perfectly matches the intuition that Fair-LinUCB is able to keep aware of the fairness and guarantee there is no reward gap between different subgroups or individuals, however, it suffers from a relatively higher regret.

\section{Experimental Evaluation}

\subsection{Experiment Setup}
\subsubsection{Simulated Dataset}
There are presently no publicly available datasets that fits our environment. We therefore generate one simulated dataset for our experiments by combining the following two publicly available datasets.
\begin{itemize}
    \item \textbf{Adult dataset}: The Adult dataset \cite{Dua:2019} is used to represent the students (or bandit players). It is composed of 31,561 instances: 21,790 males and 10,771 females, each having 8 categorical variables (work class, education, marital status, occupation, relationship, race, sex, native-country) and 3 continuous variables (age, education number, hours per week), yielding an overall of 107 features after one-hot encoding. 
    \item \textbf{YouTube dataset}: The Statistics and Social Network of YouTube Videos \footnote{https://netsg.cs.sfu.ca/youtubedata/} dataset is used to represent the items to be recommended (or arms). It is composed of 1,580 instances each having 6 categorical features (age of video, length of video, number of views, rate, ratings, number of comments), yielding a total of 25 features after one-hot encoding. We add a 26\textsuperscript{th} feature used to represent the gender of the speaker in the video which is drawn from a Bernoulli distribution with  the probability of success as 0.5.
\end{itemize}

The feature contexts  $\mathbf{x}_{t,a}$ used throughout the experiment is the concatenation of both the student feature vector and the video feature vector. In our experiments we choose the sensitive attribute to be the \textbf{gender of adults}, and we therefore focus on the unfairness on the group-level for the male group and female group. Furthermore, we assume that a male student prefers a video featuring a male and a female student prefers a video featuring a female speaker. Thus, in order to maintain the linear assumption of the reward function, we add an extra binary variable in the feature context vector that represents whether or not the gender of the student matches the gender of the speaker in the video. Overall, $\mathbf{x}_{t,a}$ contains a total of 134 features.

For our experiments, we use a subset of 5,000 random instances from the Adult dataset, which is then split into two subsets: one for training and one for testing. The training subset is composed of 1,500 male individuals and 1,500 female individuals whilst the testing subset is composed of 1000 males and 1000 females. Similarly, a subset of YouTube dataset is used as our pool of videos to recommend (or arms). The  subset contains 30 videos featuring a male speaker and 70 videos featuring a female speaker.

\subsubsection{Reward Function}
We compare our Fair-LinUCB against the original LinUCB using a simple reward function wherein we manually set the $\bm{\theta}^*$ coefficients. The reward $r$ is defined as  
\begin{align*}
    r &= \theta^*_1 \cdot x_1 +
    \theta^*_2 \cdot x_2 + \theta^*_3 \cdot x_3
\end{align*}
where $\theta^*_1 = 0.3$, $\theta^*_2 = 0.4$, $\theta^*_3 = 0.3$ and $x_1 = \text{video rating}$, $x_2 = \text{education level}$, $x_3 = \text{gender match}$. The remaining $d-3$ coefficients are set to 0. Hence, only these three features matter to generate our true reward. The gender match is set to 1 if both the student gender and the gender of the video match, and 0 otherwise. The education level is divided into 5 subgroups each represented by a value ranging from 0.0 to 1.0 with a higher education level yielding a higher value. In our setup, the education level is used to represent the strength of the student. Similarly, the video rating varies from 0 to 1.0, and is used to represent the educational quality of the video. Evidently, a higher reward is generated when the gender of the student matches the gender of the video. 

\subsubsection{Evaluation Metrics}
Throughout our experiments we measure the effectiveness of the algorithms through the average utility loss. Since we know the true reward function, we can derive the optimal reward at each round $t$. We can thus define \[\text{utility loss} = \frac{1}{T}\sum_{t=1}^{T}(r_{t,a^*} - r_{t,a})\] 
where $r_{t,a^*}$ is the optimal reward at round $t$ by choosing arm $a^*$ and $r_{t,a}$ is the observed reward by the algorithm after picking arm $a$.

We measure the fairness of the algorithms through the absolute value of the difference between the cumulative mean reward ($\bar{r}_t$, as introduced in Section \ref{prob}) of the male group and female group:
\[ \text{reward difference} = |\bar{r}_t^{s^+}-\bar{r}_t^{s^-}| \]
Additionally, for all following figures the left hand side plots the cumulative mean reward during the training phase whilst the right hand side reflects the cumulative mean reward over the testing dataset. Due to space limit, all tables report measures on the testing data solely. Note that the contextual bandit continues to learn throughout both phases. 

\subsubsection{Baselines}
As existing fair bandits algorithms focus on item-side fairness, we mainly compare our Fair-LinUCB against LinUCB in terms of utility-fairness trade-off in our evaluations. We also report a comparison with a simple fair LinUCB method that  suppresses the unfairness by removing  the sensitive attribute and all its correlated attributes from the context. We name this method as Naive in our evaluation.

\subsection{Comparison with Baselines}

\subsubsection{Comparison with LinUCB}
Our first experiment compares the performances of the traditional LinUCB against our Fair-LinUCB, using the reward function $r$ described in the previous section. Figure \ref{exp1} plots the cumulative mean reward of both the male and female groups over time. We can notice that the cumulative mean rewards of both groups suffer a discrepancy with LinUCB, and the outcome can therefore be considered unfair towards the male group. Indeed, as shown on Figure \ref{LinUCB1} the cumulative mean reward of the female group (0.839) is greater than the cumulative mean reward of the male group (0.802), yielding a reward difference of 0.037. The utility loss incurred  is 0.050. In contrast, Fair-LinUCB  is able to seal the reward discrepancy with a $\gamma$ coefficient set to 3 (Figure \ref{fairCB1}). Our algorithm thereby achieves a cumulative mean reward of 0.819 for both the male group and the female group, which yields a reward difference of 0.0, while incurring a utility loss of 0.052. Our Fair-LinUCB  outperforms the traditional LinUCB in terms of reward difference while suffering a slight loss of utility. The comparison results are summarized in the first two rows of Table \ref{corr_features}. 

\begin{table}
\centering
\caption{Comparison of Three Algorithms under Reward Function $r$}
\begin{tabular}{c||c|c} 
    & Utility Loss & Reward difference \\
    \hline
    \hline
Fair-LinUCB ($\gamma$ = 3) & 0.052 & \textbf{0.000}\\ \hline    
LinUCB & 0.050 & 0.037\\ 
    \hline 
Naive & \textbf{0.046} & 0.035\\ 
\end{tabular}
\label{corr_features}
\end{table}

To evaluate how the inclusion or exclusion of sensitive attributes affects the fairness-utility tradeoff, we compare LinUCB against Fair-LinUCB with a modified reward function:
\begin{align*}
     r_2 = \theta^*_1 \cdot x_1 + \theta^*_2 \cdot x_2
\end{align*}
where $\theta^*_1 = 0.5$ and $\theta^*_2 = 0.5$ and $x_1 = \text{video rating}$, $x_2 = \text{education level}$
The remaining $d-2$ coefficients are set to 0.  $r_2$ is not dependent upon the gender match attribute and expects to incur zero or small discrepancy between both groups. As depicted on Figure \ref{exp2}, both LinUCB and Fair-LinUCB show a very low cumulative mean reward discrepancy. Specifically, LinUCB incurs a utility loss of 0.037 and a reward difference of 0.006, while Fair-LinUCB incurs 0.034 utility loss and a reward difference of 0.008. Furthermore, in this case, although Fair-LinUCB has additional constraints for the arm picking strategy due to the fairness penalty, it does not induce any loss of utility when compared to LinUCB.  

\begin{figure*}[th!]
    \centering
    \begin{subfigure}[t]{0.5\textwidth}
        \centering
        \includegraphics[width=2.7in]{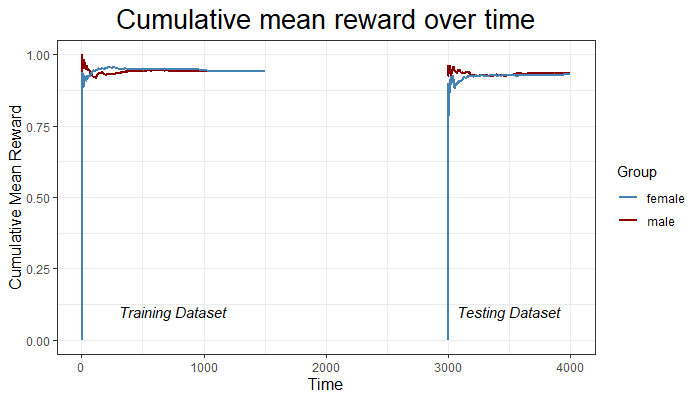}
        \caption{LinUCB}
        \label{LinUCB2}
    \end{subfigure}%
    ~ 
    \begin{subfigure}[t]{0.5\textwidth}
        \centering
        \includegraphics[width=2.7in]{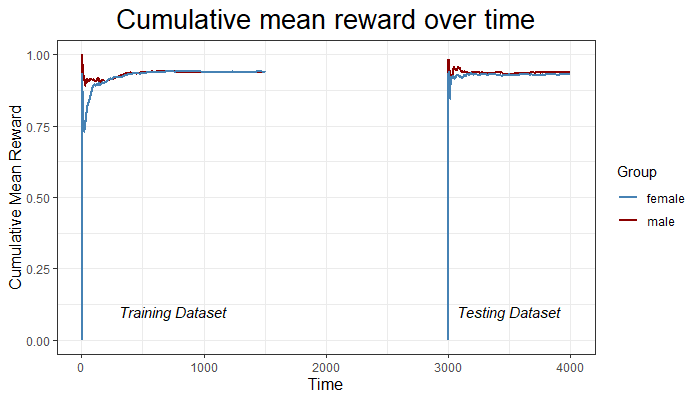}
        \caption{Fair-LinUCB $\gamma$ = 3}
        \label{fairCB2}
    \end{subfigure}
    \caption{LinUCB vs Fair-LinUCB with reward function $r_2$}
    \label{exp2}
\end{figure*}

\subsubsection{Comparison with Naive}

Naive method tries to achieve fairness by removing from the context the sensitive attribute and the features that are highly correlated with the sensitive attribute. In our experiment, we first compute the correlation matrix of all the user's features and then remove the gender feature as well as all features that are highly correlated with it. Specifically, features that have a correlation coefficient greater than 0.3 were removed, which include the following: is male, is female, is divorced, is married, is widowed, is a husband, has an administrative clerical job, has a salary less than 50k. 
We report in the last row of Table \ref{corr_features} the utility loss and reward difference of Naive with reward function $r$. 

We can see  the reward discrepancy between the male and female groups from the Naive method is 0.035, thus showing it cannot  completely remove discrimination.  The utility loss from the Naive method is 0.046, which is only slightly smaller than  LinUCB and Fair-LinUCB. In fact, as shown in Table \ref{fairness_tradeoff}, Fair-LinUCB with $\gamma=2$ can outperform the Naive method in terms of both fairness and utility. In short, removing the gender information and highly correlated features from the context does not necessarily close the gap of the reward difference. 

In summary, although LinUCB learns to pick the arm that maximizes the reward given a particular context, we have seen that it could incur discrimination towards a group of users in some cases. 
Fair-LinUCB is capable of detecting when unfairness occurs, and will adapt its arm picking strategy accordingly so as to be as fair as possible and reduce any reward discrepancy. When a reward discrepancy is not detected, our algorithm does not need to adjust the arm picking strategy and therefore performs as well as the traditional LinUCB.
\begin{figure*}[th!]
    \centering
    \begin{subfigure}[t]{0.5\textwidth}
        \centering
        \includegraphics[width=2.7in]{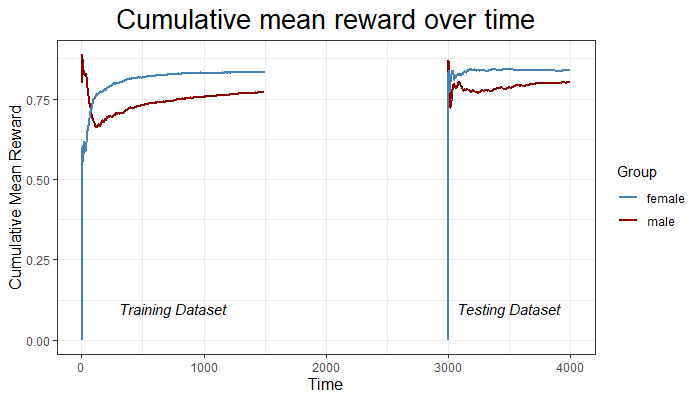}
        \caption{LinUCB}
        \label{LinUCB1}
    \end{subfigure}%
    ~ 
    \begin{subfigure}[t]{0.5\textwidth}
        \centering
        \includegraphics[width=2.7in]{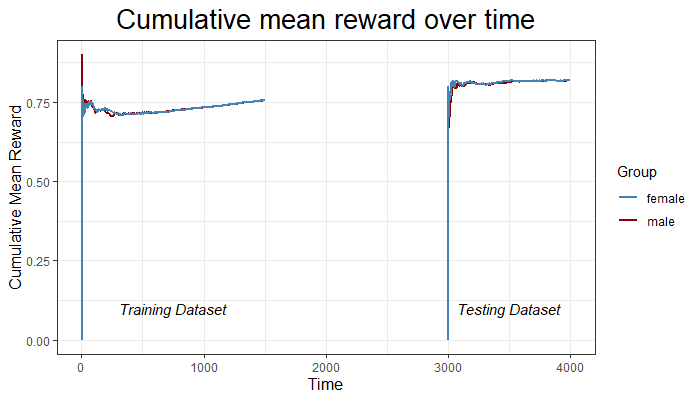}
        \caption{Fair-LinUCB $\gamma$ = 3}
        \label{fairCB1}
    \end{subfigure}
    \caption{LinUCB vs Fair-LinUCB with reward function $r$}
    \label{exp1}
\end{figure*}

\subsection{Impact of  $\gamma$ on Fairness-Utility Trade-off}
The $\gamma$ coefficient introduced in Section \ref{fairnessCB} controls the weight of the fairness penalty that the algorithm will exert onto the UCB value. Indeed, as shown in Equation \eqref{mapping_function}, $\gamma$ is used to adjust the upper bound of the linear mapping function $\mathcal{L}(\gamma, F_a)$. Thus, when the $\gamma$ coefficient increases, the range of the fairness penalty increases proportionally which will consequently increase the UCB value in Equation \ref{fairpta}. The $\gamma$ coefficient therefore reflects the significance of the fairness of Fair-LinUCB. However, as $\gamma$ becomes larger, the fairness penalty becomes out of proportion to the extent of neglecting the importance of the UCB value, thereby decreasing the utility of the algorithm. 
\begin{table}
\centering
\caption{Impact of $\gamma$ on the Fairness-Utility Trade-off}
\begin{tabular}{c||c|c} 
    & Utility Loss & Reward difference \\
    \hline
    \hline
$\gamma = 0$ & 0.050 & 0.037\\ 
    \hline
$\gamma = 1$ & 0.040 & 0.016\\ 
    \hline
$\gamma = 2$ & \textbf{0.035} & 0.004\\ 
    \hline
$\gamma = 3$ & 0.052 & \textbf{0.000}\\ 
    \hline 
$\gamma = 4$ & 0.081 & \textbf{0.000}\\
\end{tabular}
\label{fairness_tradeoff}
\end{table}

To evaluate the fairness-utility trade-off of Fair-LinUCB, we compare several $\gamma$ values and report the fairness and utility loss in Table \ref{fairness_tradeoff}. With a $\gamma$ equal to 0, our algorithm behaves as a traditional LinUCB, therefore it incurs discrimination (reward difference measured at 0.037), and a utility loss of 0.050 is reported. We can observe that when $\gamma$ increases slightly, the algorithm improves the reward difference and loss of utility. Specifically, a reward difference of 0.016 is achieved for $\gamma$ = 1 with a utility loss of 0.040, and a reward difference of 0.004 with a utility loss of 0.035 is achieved with $\gamma$ = 2. Although the utility losses are improved, they both remain not fair. In our best case scenario, with $\gamma$ = 3, the algorithm is completely fair, i.e., reward difference is 0.000, with a utility loss of 0.052. Finally, when the $\gamma$ coefficient is too large, the algorithm prioritizes fairness over utility, resulting in a fair algorithm that suffers a greater loss of utility. For example, with a $\gamma$ set to 4, Fair-LinUCB incurs a utility loss of 0.081. 

\subsection{Impact of arm and user distributions}

In certain cases the distribution of the arms (videos) or the users can significantly impact the cumulative mean reward of some groups of users, and therefore incur the large reward difference. In our experiment, given the reward function $r$, we first explore the impact of the ratio of gender arms, i.e., videos by female or male speakers, and then we investigate the impact of the order of the data in which the algorithm learns. The following results discuss our findings.
\subsubsection{Gender arm ratio}
We explore the effect of three different arm ratio values: (1) 70\% male and 30\% female, (2) 50\% male and 50\% female, and (3) 30\% male and 70\% female. Table \ref{ucbarmratio} reports the utility loss, reward difference, as well as both the cumulative mean reward for the male and female groups. As observed with the LinUCB performances, the arm ratio induces unfairness on some user group.
Indeed, when there is a majority of male arms, it appears that the male user group will benefit more and will have a higher cumulative mean reward. Likewise, when the arms have more females than males, the female user group will benefit more than the male user group, and will therefore have a higher cumulative mean reward. Although having a balanced ratio of male and female arms minimizes the reward difference, it is not always feasible or convenient to adjust the arms distribution in practice.

We ran the same experiment with Fair-LinUCB with $\gamma$ = 3. As we can see, in all three cases, Fair-LinUCB yields a very low reward difference. Indeed, our Fair-LinUCB learns which group is being discriminated and adjusts its arm picking strategy accordingly so as to remove any discrimination, it however suffers a higher utility loss than LinUCB. Note that a $\gamma$ different than 3 could yield a better utility loss for the ratios 7:3 and 1:1. 

Thus, as opposed to a traditional LinUCB which only learns to maximize the reward given a context, our Fair-LinUCB learns how to achieve fairness at the same time, making  it  robust against factors that would otherwise induce unfairness.
\begin{table}
\centering
\caption{ Impact of different arm ratio on the fairness and utility}
\begin{tabular}{c||c|c|c|c} 
    \multirow{2}{*}{Arm ratio} & \multirow{2}{*}{Utility Loss} & Reward & Male & Female\\
    m:f & & difference & cmr & cmr\\
    \hline
    \multicolumn{5}{c}{LinUCB}\\
    \hline
    \hline
  7:3 & 0.061 & 0.029 & 0.824 & 0.795\\
    \hline
  1:1 & 0.053 & 0.012 & 0.824 & 0.812\\ 
    \hline
  3:7 & \textbf{0.050} & 0.037 & 0.802 & 0.839\\  
    \hline
    \hline
    \multicolumn{5}{c}{Fair-LinUCB $\gamma = 3$}\\
    \hline
    \hline
  7:3 & 0.087 & 0.001 & 0.784 & 0.783\\ 
    \hline
  1:1 & 0.162 & \textbf{0.000} & 0.709 & 0.709\\ 
    \hline
  3:7 & 0.052 & \textbf{0.000} & 0.819 & 0.819\\ 
\end{tabular}
\label{ucbarmratio}
\end{table}

\subsubsection{Order of the training data} 
\begin{figure*}[th!]
    \centering
        \begin{subfigure}[t]{0.4\textwidth}
        \centering
        \includegraphics[height=4cm]{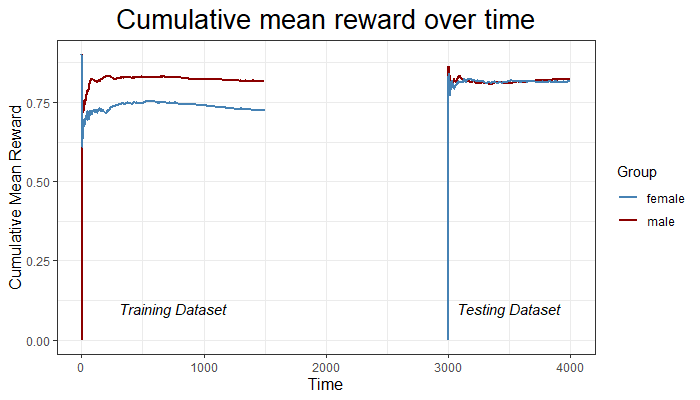}
        \caption{1,500 females followed by 1,000 males}
        \label{order_dataa}
    \end{subfigure}
    ~ 
    \begin{subfigure}[t]{0.4\textwidth}
        \centering
        \includegraphics[height=4cm]{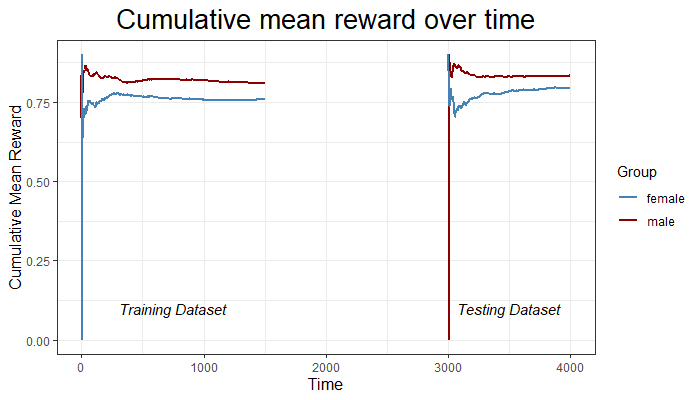}
        \caption{1,500 males followed by 1,000 females}
        \label{order_datab}
    \end{subfigure}
    \caption*{LinUCB}
    ~ 
    \begin{subfigure}[t]{0.4\textwidth}
        \centering
        \includegraphics[height=4cm]{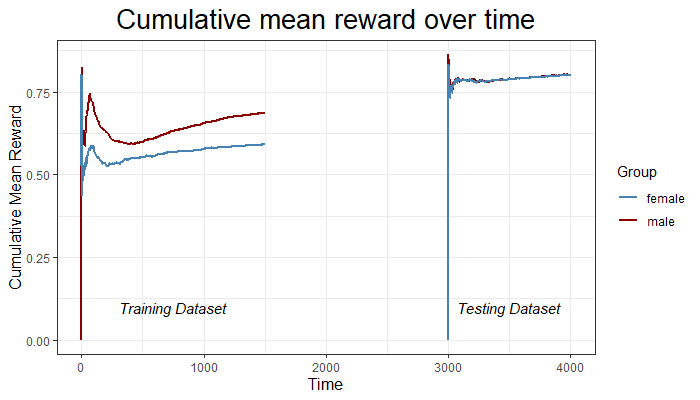}
        \caption{1,500 females followed by 1,000 males}
        \label{order_datac}
    \end{subfigure}
    ~ 
    \begin{subfigure}[t]{0.4\textwidth}
        \centering
        \includegraphics[height=4cm]{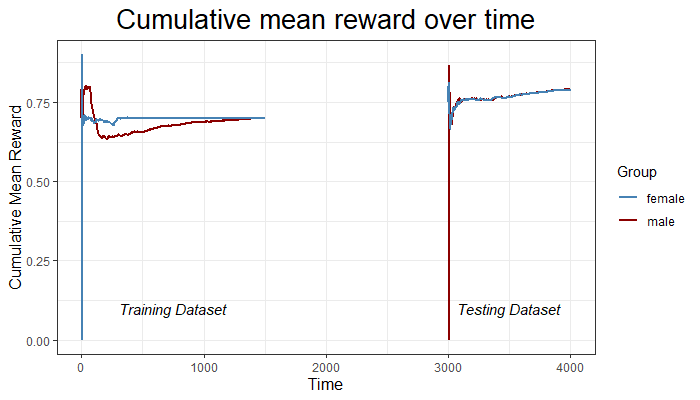}
        \caption{1,500 males followed by 1,000 females}
        \label{order_datad}
    \end{subfigure}
    \caption*{Fair-LinUCB $\gamma$ = 3}
    \caption{Impact of the order of the data on the performances}
    \label{order_data}
\end{figure*}
It is our intuition that the order of the data in which LinUCB learns to recommend an item could affect its recommendation choice or arm pick. 

In these experiments, we use the 70\% male and 30\% female arms setting, and we manually change the order of the training data. In the first setting, we manually set the order of the students in the training data by having all 1,500 female students followed by the 1,500 males instances. In the second setting we order the data by having all 1,500 male instances first, followed by the 1,500 female instances. The test data remains shuffled. We then compare LinUCB with Fair-LinUCB  in order to see the impact on the learning strategy of both algorithms.

We ran the traditional LinUCB and report the cumulative mean reward of the male user group and female user group over time. As shown in Figure \ref{order_dataa} and Figure \ref{order_datab}, overall the male group gets a higher cumulative mean reward than the female group. Particularly, the male group achieves 0.822 against 0.816 for the female group in Figure \ref{order_dataa} and 0.834 against 0.795 in Figure \ref{order_datab}. However, we notice that the reward discrepancy is much higher in the second scenario as compared to the first one. From Figure \ref{order_dataa}, it appears that learning to recommend videos to all females students prior to recommending videos to any male students affects the recommendation process positively (i.e., it yields a higher cumulative mean reward for the female group). Thus, the order of the training data can sometimes affect the recommendation process of LinUCB, which can impact the recommendation outcomes and may also induce discrimination towards one group.

We ran the same experiments with Fair-LinUCB, using a $\gamma$ coefficient of 3, and we report our results in Figure \ref{order_datac} and Figure \ref{order_datad}.  We notice that in both situations our Fair-LinUCB remains very fair, that is, we do not observe a cumulative mean reward discrepancy between the male and female user group. In the former setting, both groups achieve a cumulative mean reward of 0.802 against 0.789 in the latter, both yielding a cumulative mean reward difference of 0.00. In addition, we notice that regardless of the order of the training data our Fair-LinUCB performs equivalently in both scenarios. However, the gain in fairness also induces a loss of utility. Indeed, in the first setting LinUCB achieves 0.052 utility loss against 0.070 for Fair-LinUCB. In the second setting, LinUCB achieves 0.057 against 0.082 for Fair-LinUCB. Thus, our results indicate that Fair-LinUCB is able to close the reward discrepancy and is robust against scenarios that might otherwise induce unfairness.

\section{Conclusion}
Previous research have shown that personalized recommendation can be highly effective at a cost of introducing unfairness. In this paper, we have proposed a fair contextual bandit algorithm for personalized recommendation. While current research in fair recommendation mainly focus on how to achieve fairness on the items that are being recommended, our work differs by focusing on fairness on the individuals whom are being recommended an item. Specifically, we aim to recommend items to users while insuring that both the protected group and privileged group improve their learning performance equally. Our developed Fair-LinUCB improves upon the state-of-the-art LinUCB algorithm by automatically detecting unfairness, and adjusting its arm-picking strategy such that it maximizes the fairness outcome. We further provided a regret analysis of our fair contextual bandit algorithm and demonstrate that the regret bound is only worse than LinUCB up to an additive constant. Finally, we evaluate the performances of our Fair-LinUCB  against that of LinUCB by comparing both their effectiveness and degree of fairness. Experimental evaluations showed that our Fair-LinUCB achieves competitive effectiveness while outperforming LinUCB in terms of fairness. We further showed that our algorithm is robust against numerous factors that would otherwise induce or increase discrimination in the traditional LinUCB algorithm. In this work we made a linear assumption on the reward function. In the future work, we plan to extend the user-level fairness to more general cases and make it easier to be implemented in multifarious reward functions. We plan to develop heuristics to determine the appropriate value for the fairness-accuracy trade off parameter $\gamma$. We also plan to study user-side fairness in the multiple choice linear bandits, e.g., recommending multiple videos to a student at each round. Finally, we plan to study how to achieve individual fairness in bandits algorithms. 

\section{Acknowledgement}
This work was supported in part by NSF 1937010, 1940093, 1940076, and 1940236.\\

{\bf \noindent Reproducibility}.
The source code and datasets are available at \url{https://www.dropbox.com/s/44bwtnxs0j8wbw4/Achieving_User-Side_Fairness_in_Contextual_Bandits.zip?dl=0}.


\end{document}